\documentclass{article}
\usepackage{fullpage}
\usepackage{authblk}
\usepackage{mathtools}
\usepackage{amsmath}
\usepackage{amsthm}
\usepackage{amssymb}
\usepackage{amsmath}
\usepackage{cite}
\usepackage{booktabs}

\usepackage{pstricks}
\usepackage{tikz}
\usepackage{pgfplots}
\pgfplotsset{compat=newest}
\usepgfplotslibrary{groupplots}
\usepgfplotslibrary{dateplot}
 
\newtheorem{theorem}{Theorem}
\newtheorem{lemma}[theorem]{Lemma}

\newtheorem{observation}[theorem]{Observation}

\newtheorem{corollary}[theorem]{Corollary}

\newcommand{\norm}[1]{\left\lVert#1\right\rVert}

\newcommand{\eps}{\epsilon}

\usepackage[utf8]{inputenc} 
\usepackage[T1]{fontenc}    
\usepackage{hyperref}       
\usepackage{url}            
\usepackage{booktabs}       
\usepackage{amsfonts}       
\usepackage{nicefrac}       
\usepackage{microtype}      

\usepackage{graphicx}
\usepackage{subcaption}

\title{On Coresets for Regularized Loss Minimization}

\author[1]{Ryan Curtin}
\affil[1]{
  RelationalAI\\
  Atlanta, GA 30318 \\
  \texttt{ryan@ratml.org} }
  \author[2]{Sungjin Im\thanks{Supported in part by NSF grants CCF-1409130 and CCF-1617653} }
   \affil[2]{Computer Science and Engineering \\
    UC Merced   \\
   \texttt{sim3@ucmerced.edu} }
   \author[3]{Benjamin Moseley\thanks{Supported in part by NSF grants CCF-1845146, CCF-1824303, CCF-1830711, CCF-1733873, a Google faculty award and an Infor faculty award.} }
   \affil[2]{Tepper School of Business \\
   Carnegie Mellon University\\
   and RelationalAI\\
   \texttt{moseleyb@andrew.cmu.edu} }
   \author[4]{
   Kirk Pruhs\thanks{Supported in part by NSF grants CCF-1421508 and CCF-1535755, and an IBM Faculty Award.} }
   \affil[4]{Computer Science Department \\
   University of Pittsburgh \\
   \texttt{kirk@cs.pitt.edu} }
   \author[5]{
   Alireza Samadian}
   \affil[5]{Computer Science Department\\
   University of Pittsburgh \\
   \texttt{samadian@cs.pitt.edu} 
}

\begin{document}

\maketitle

\begin{abstract}
 We  design and mathematically analyze sampling-based algorithms for regularized loss minimization  problems  that are implementable in popular computational models for large data, in which the access to the data is restricted
 in some way. Our main result is that if the regularizer's effect
 does not become negligible as the norm of the hypothesis 
 scales, and as the data scales, then a
uniform sample of modest size is with high probability a coreset.
 In the  case that the loss function is either
logistic regression or soft-margin support vector machines,
and the regularizer is one of the common recommended choices,
this result implies that a uniform sample of size $O(d \sqrt{n})$ is with high probability a coreset of $n$ points in $\Re^d$. 
We contrast this 
upper bound with two lower bounds. The first lower bound shows that our analysis of uniform sampling is tight; that is,  a smaller uniform sample will likely not be a core set.
The second lower bound shows
that in some sense uniform sampling is close to optimal, 
as significantly smaller core sets do not generally exist.
\end{abstract}

\section{Introduction}

We consider the design and mathematical analysis of sampling-based algorithms for regularized loss minimization (RLM) problems on large data sets~\cite{Shalev-Shwartz:2014}.
The input  consists of a collection $X = \{x_1,x_2,\dots,x_n\}$ of points in $\Re^d$, 
and a collection 
$Y = \{y_1,y_2,\dots,y_n\}$
of  associated labels from $ \{-1,1\}$.
Intuitively  the goal is to find a hypothesis $\beta \in \Re^d$ that  is the best ``linear'' explanation
for the labels. More formally, the objective 
is to  minimize a function $F(\beta)$ that is a linear combination of a nonnegative nondecreasing loss  function $\ell$ that measures the goodness of the hypothesis, and a nonnegative regularization function $r$ that measures the complexity of the hypothesis. So:
\begin{align}
\label{General_Cost_Function}
    F(\beta) = \sum_{i=1}^n \ell(-y_i \beta \cdot x_i ) + \lambda  \;  r(R \beta)
\end{align}
Notable examples include regularized logistic 
regression, where the loss function is $\ell(z) = \log(1+\exp(z))$, and regularized soft margin  support vector machines (SVM), where the loss function is $\ell(z) = \max (0, 1 +z)$.
Common regularizers are the 1-norm, the 2-norm, and
the 2-norm squared~\cite{Bhlmann:2011}. 
The parameter $\lambda\in \Re$ is ideally set to balance the risks of over-fitting and under-fitting. 
  It is commonly recommended to set $\lambda$ to be proportional to $\sqrt{n}$~\cite{Shalev-Shwartz:2014,NegahbanRWY09}. For this choice of $\lambda$, if there was a true underlying distribution from which the data was drawn in an i.i.d. manner, then there is a guarantee that the computed $\beta$ will likely have vanishing relative error with respect to the ground truth~\cite[Corollary 13.9]{Shalev-Shwartz:2014}~\cite[Corollary 3]{NegahbanRWY09}. 
  We will generalize this somewhat and assume that $\lambda$ is proportional to $n^\kappa$ for some $0 < \kappa < 1$.
  The parameter $R$ is the maximum 
  2-norm of any point in $X$. 
 Note that the regularizer must scale with $R$ if it is to
 avoid having a vanishing effect as the point set $X$ scales.\footnote{To see this note that if we multiplied each coordinate of each point
 $x_i$ by a factor of $c$, the optimal hypothesis $\beta$ would decrease by a factor of $c$, thus decreasing the value of all of the
 standard regularizers.}

 We are particularly interested in settings where the data set is too large to fit within the main memory of one computer, and thus the algorithm's access to the data set is restricted in some way. Popular computation models that arise from such
 settings include:
\begin{description}
\item[Streaming Model:]
This model derives from settings where the data is generated in real-time, or stored on a memory technology (such as a disk or tape) where a sequential scan is  way more efficient than random accesses. In this model the data can only be accessed by a single (or a small number of)  sequential passes~\cite{Muthukrishnan:2005}.
\item[Massively Parallel Computation (MPC) Model:]
This model derives from settings where the data is distributed over multiple computers. In this model only a few rounds of communication with sublinear sized messages are allowed~\cite{ImMS17,KarloffSV10}.
\item[Relational Model:] This model derives from settings where the data is stored in a database in a collection of tables. In this model the data must be accessed via relational operators that do not explicitly join tables~\cite{KhamisNR16}.
\end{description}
Thus we additionally seek algorithms that can be reasonably implemented in these popular restricted access models.

One popular method to deal with large data sets is to extract a manageably small  (potentially weighted) sample from the data set,
and then directly 
solve (a weighted version of) the RLM problem on the  (weighted) sample\footnote{A more general approach is to summarize the data set in some more sophisticated way than as a weighted sample, but such approaches are beyond the scope of this paper.}.
The aspiration here is that the optimal solution on the sample will be a good approximation to the optimal solution on the original data set. 
To achieve this aspiration, the probability that a particular
point is sampled (and the weight that it is given)
may need to be carefully computed as  some points may be more important than other points. But if this sampling
probability distribution is too complicated, 
it may not be efficiently implementable in common restricted access models.

A particularly strong condition on the sample that is sufficient for achieving this 
aspiration is that the sample  is a \emph{coreset};
intuitively, a sample is a coreset if
\emph{for all possible hypotheses} $\beta$, the objective value of $\beta$ on the sample is very close to the objective value of $\beta$ on the whole data set.

Our original research goal was to determine when small coresets exist for RLM problems in general, and for regularized logistic
regression and regularized SVM in particular,
and when these coresets can be efficiently computed 
within the common restricted access models.

\noindent \textbf{Our Results:} 
Our main result, covered in Section \ref{sec:ub}, is that if the regularizer's effect
 does not become negligible as the norm of the hypothesis  scales then a
uniform sample of  size {\small $\Theta(n^{1-\kappa} \Delta)$} points is with high probability a coreset. Here, $\Delta$ is the VC-dimension of the loss function.
Formally this scaling condition 
says that if $\ell(-\norm{\beta}) = 0$ then
$r(\beta)$
must be a constant fraction of 
$\ell(\norm{\beta}_2)$.
We show that this scaling
condition holds when the loss function is either
logistic regression or SVM, and the regularizer is the
1-norm, the 2-norm, or the 2-norm squared. 
So for example, in the standard case that $\kappa = 1/2$, the scaling
condition ensures that a uniform sample of {\small $\tilde \Theta(d \sqrt{n} )$}
points is with high probability a coreset when the regularizer is one of the standard ones,
and the loss function is either
logistic regression and SVM, as they have VC-dimension $O(d)$.
Note also that uniform sampling can be reasonably implemented in all of the
popular restricted access models. So this yields a reasonable algorithm
for all of the restricted access models under the assumption that a data set of size
{\small $\tilde \Theta(d \sqrt{n} )$} can be stored, and reasonably solved, in the main memory of one computer.

We complement our upper bound  with two lower bounds on the size of coresets. Our lower bounds assume the 2-norm squared as the
regularizer, since intuitively this is the standard regularizer
for which it should be easiest to attain small coresets.
We first show in Section \ref{sec:lbuniform} that our analysis is asymptotically  tight for uniform sampling. That is, we show that for both logistic regression and SVM, 
a uniform sample of size {\small $O(n^{1-\kappa -\epsilon})$} may not result in a coreset.  We then show in Section \ref{sec:lbgeneral} that for
both logistic regression and SVM there are instances
in which every core set is of size {\small $\Omega(n^{(1 - \kappa)/5 - \epsilon})$}.
So more sophisticated 
sampling methods must still have core sets whose size is in the same 
ballpark as is needed for uniform sampling. 
One might arguably summarize our results as 
saying that the simplest possible sampling method
is nearly optimal for obtaining a coreset.


Finally in Section \ref{sec:experiments} we  experimentally evaluate the practical utility of  uniform sampling for logistic regression
using several real-world datasets from the UCI machine learning dataset repository~\cite{ucimlrepository}.
We observe that our theory is empirically validated as uniform samples yield good empirical approximation, 
and orders-of-magnitude speedup over learning on the full dataset.

 \noindent \textbf{Related Work on Coresets:}
 The most closely related prior work is probably \cite{MunteanuSSW18}, who considered coresets for \emph{unregularized} logistic regression; i.e, the regularization
parameter $\lambda=0$. \cite{MunteanuSSW18}  showed that are data sets for which there do not exist coresets of sublinear size,
and then introduced a parameter $\mu$ of the
instances that intuitively is small when
there is no hypothesis that is a good
explanation of the labels, 
and showed that a
 coreset of size roughly 
 linear in $\mu$ can be obtain by sampling
 each point
 with a uniform probability plus a probability proportional to its $\ell_2^2$ leverage scores (which can be computed from a singular value decomposition of the points).
This result yields an algorithm,
for the promise problem in which $\mu$ is known a priori to be small (but it is not clear how to reasonably compute $\mu$), that is reasonably implementable in the MPC model,
and with two passes over the
data in the streaming model. It seems unlikely that this
algorithm is implementable in the relational model due to the 
complex nature of required sampling probabilities. 
Contemporaneously with our research, \cite{Tolochinsky} 
obtained results similar in flavor to those of \cite{MunteanuSSW18}. \cite{Tolochinsky} also show that small coresets exist for certain types of
RLM instances; in this case, those in which the norm 
of the optimal hypothesis is small. So for normalized
logistic regression \cite{Tolochinsky} shows
that when the 2-norm of the optimal $\beta$ is bound by $\mu$,
coresets of size 
$\tilde O( \mu^2 n^{1-\kappa})$ can be obtained
by sampling a point with probability proportional to its
 norm divided by its ordinal position in the sorted order of norms. So again this yields an algorithm
for the promise problem in which $\mu$ is known a priori to be small (and again it is not clear how to reasonably compute $\mu$). 
Due to the complex nature of the probabilities it is not
clear that this algorithm is reasonably implementable in any
of the restricted access models that we consider. 
So from our perspective there are three key differences between the
results of \cite{MunteanuSSW18} and \cite{Tolochinsky} and our positive result: {\it (1)} our result applies to all data sets {\it (2)}
we use uniform sampling, and thus {\it (3)} our sampling algorithm
is implementable in all of the restricted access models
that we consider.

Surveys of the use of coresets in algorithmic design can be found in  \cite{Munteanu2018}  and  in \cite[Chapter 23]{Har-peled:2011}. 
The knowledge that sampling with probability at least proportional
to sensitivity yields a coreset has been used for at least
a decade as it is used by \cite{DasguptaDHKM09}.
Coresets were used for partitioned clustering problems, such as $k$-means~\cite{Har-Peled-2004,MeyersonOP04,BachemL018}.
Coresets were also used  the Minimum Enclosing Ball (MEB) problem~\cite{Har-peled:2011}. Coresets for MEB are the basis for the Core Vector Machine approach to  unregularized kernelized SVM~\cite{TsangKC05}.  We note that while there is a reduction from kernelized SVM to MEB,  the reduction is not approximation preserving, and thus the existence of  coresets for MEB does not necessarily imply the existence of  coresets for SVM. Coresets have also been used for submodular optimization~\cite{MirrokniZ15}, and
in the design of  streaming algorithms (e.g.~\cite{OCallaghanMMMG02}), as well as distributed algorithms (e.g.~\cite{MalkomesKCWM15}).

\section{Preliminaries}


We define 
$\ell_i(\beta) = \ell(-y_i \beta \cdot x_i )$ as the contribution of point $i$ to the loss function.
We define 
$f_i(\beta) = \ell(-y_i \beta \cdot x_i ) + \lambda r(R \beta)/n$ as the contribution of point $i$ to the objective
$F(\beta)$. The sensitivity of point $i$ is then
$s_i = \sup_{\beta} \;f_i(\beta)/F(\beta)$,
and the total sensitivity is $S = \sum_{i=1}^n s_i$.
 For  $\epsilon>0$, 
an $\epsilon$-coreset $(C, U)$ consists of a subcollection $C$ of $[1, n]$,
and associated nonnegative weights $U = \{u_i \mid i \in C \}$, 
such that
 \begin{align}
\label{coreset_inequality}
\forall \beta  \ \ \ \ \ H(\beta) \coloneqq \frac{\left|\sum_{i=1}^n  f_i(\beta)-\sum_{i \in C}  u_i f_i(\beta)
\right|}{\sum_{i=1}^n  f_i(\beta)} \leq \epsilon
\end{align}
Conceptually one should think of $u_i$ as a multiplicity,
that is that $x_i$ is representing $u_i$ points from the original
data set. So one would expect that $\sum_{i \in C} u_i = n$,
although this is not strictly required. But it is easy to observe that $\sum_{i \in C} u_i$ must be close to $n$. 

\begin{observation}
\label{obs:Un}
Assume that $\ell(0) \ne 0$, as is the case for logistic 
regression and SVM. 
If $(C, U) $ is an $\epsilon$-coreset then $(1-\epsilon) n \le \sum_{i\in C} u_i \le (1+\epsilon) n$.
\end{observation}
\begin{proof}
Applying the definition of coreset in the case that $\beta$ is the hypothesis with all 0 components, 
it must be the case that $\left| \sum_{i=1}^n \ell(0) -\sum_{i \in C} u_i \ell(0) \right| \le \epsilon \sum_{i=1}^n \ell(0)$, or
equivalently 
$\left| n -\sum_{i \in C} u_i \right| \le \epsilon n$.
\end{proof}

Note that in the special case that each $u_i$ is equal to a common value $u$, as will be the case for uniform sampling, setting each $u_i=1$ and scaling $\lambda$ down by a factor of $u$, would result in the same optimal hypothesis $\beta$.

A collection $X$ of data points is {\it shatterable} by a loss
function $\ell$ if
for every possible set of assignments of labels, there
is a hypothesis $\beta$ 
and a threshold $t$, such that
for the positively labeled points $x_i \in X$ it is the case the
$\ell( \beta \cdot x_i)  \ge t$, and for the negatively labeled points
$x_i$ it is the case that $\ell( \beta \cdot x_i) < t$. The VC-dimension of a loss function
is then the maximum cardinality of a shatterable set. 
It is well known that if the loci of points $x \in \Re^d$ where 
$\ell( \beta \cdot x )  = t$ is a hyperplane then
the VC-dimension is at most $d+1$~\cite{Vapnik1998}.
It is obvious that this property holds if the loss function
is SVM, and \cite{Munteanu2018} show that it holds if the
loss function is logistic regression. The regularizer does
not affect the VC-dimension of a RLM problem.

A loss function $\ell$  and a regularizer $r$ satisfy the 
$(\sigma, \tau)$-scaling  condition if $\ell(-\sigma) > 0$, 
and  if $\norm{\beta}_2 \ge \sigma$ then
 $r(\beta) \ge \tau \; \ell(\norm{\beta}_2)$.

\begin{theorem}[\cite{FeldmanL11,BravermanFL16}]
\label{Coreset_Construction_Theorem}
Let $(n, X, Y, \ell, r, \lambda, R, \kappa)$ be an instance
    of the RLM problem where  the
    loss function has VC-dimension at most $\Delta$.
Let  $s'_i$ be an upper bound on the sensitivity $s_i$,
let  $S' = \sum_P s'_i$. Let  $\epsilon , \delta \in (0,1)$ be 
arbitrary. Let $C \subseteq [1, n]$ be a random sample of at least
$
    \frac{10S'}{\epsilon^2}(\Delta \log{S'} + \log(\frac{1}{\delta})))
$
points sampled in an i.i.d fashion, where the probability that
point $i\in [1,n]$ is selected each time is $s'_i/S'$. Let the associated weight $u_i$
for each point $i \in C$ be  $ \frac{S' }{s'_i \; |C|}$. Then  $C$ and $U=\{u_i \mid i \in C \}$  is an $\epsilon$-coreset with probability at least $(1-\delta)$,.
\end{theorem}

\section{Upper Bound for Uniform Sampling}
\label{sec:ub}

\begin{theorem}
    Let $(n, X, Y, \ell, r, \lambda, R, \kappa)$ be an instance
    of the RLM problem where $\ell$ and $r$
    satisfy the $(\sigma, \tau)$-scaling condition and the
    loss function has VC-dimension at most $\Delta$.
    Let $S' = \frac{n}{\tau \lambda} + \frac{\ell(\sigma)}{ \ell(-\sigma)} + 1$.  A uniform sample of  
    $q = \frac{10S'}{ \epsilon^2}(\Delta \log{S'} + \log(\frac{1}{\delta}))$ points, each with an associated
    weight of  $u = n/q$, is an $\epsilon$-coreset with probability at least $1-\delta$.
\end{theorem}
\begin{proof}
With an aim towards applying Theorem \ref{Coreset_Construction_Theorem} we start by upper bounding
the sensitivity of an arbitrary point.
To this end consider an arbitrary  $i \in [1, n]$ and an arbitrary hypothesis $\beta$. First consider the case that $R \norm{\beta}_2  \geq \sigma$. 
   In this case:
    \begin{align*}
        \frac{f_i(\beta)}{F(\beta)} &= \frac{\ell(-y_i \beta \cdot x_i ) + \frac{\lambda}{n}  r(R\beta)}{\sum_j \ell(-y_j \beta \cdot x_j )) + \lambda \; r(R \beta)} & 
        \\
        &\leq \frac{\ell(|\beta \cdot x_i| ) + \frac{\lambda}{n}  r(R \beta)}
        {\sum_j \ell(-y_j \beta \cdot x_j ) + \lambda \; r(R\beta)}&\mbox{As the loss function is nondecreasing}
        \\
        &\leq \frac{\ell(|\beta \cdot x_i| ) + \frac{\lambda}{n}  r(R \beta)}
        {\lambda \; r(R\beta)}&\mbox{As the loss function is nonnegative}
        \\
         &\leq \frac{\ell(|\beta \cdot \beta | \frac{R}{\norm{\beta}_2} ) + \frac{\lambda}{n}  r(R \beta)}
        {\lambda \; r(R \beta)}& \mbox{As maximum is when } x_i = \beta \frac{R}{\norm{\beta}_2} \\
        &\leq 
        \frac{\ell(R \; \norm{\beta}_2  )}
        {\lambda \; r(R\beta)} + \frac{1}{n}&
        \\
        &\leq 
        \frac{\ell(R \; \norm{\beta}_2  )}
        {\lambda \; \tau \; \ell( R \norm{\beta}_2)} + \frac{1}{n}&\mbox{By } (\sigma, \tau)\mbox{ scaling assumption and assumption }  R \norm{\beta}_2  \geq \sigma
        \\
        &\leq 
        \frac{1}
        {\tau \lambda} + \frac{1}{n}
    \end{align*}

Next consider the case that $R\norm{\beta}_2 < \sigma$.
In this case:
\begin{align*}
        \frac{f_i(\beta)}{F(\beta)} &= \frac{\ell(-y_i \beta \cdot x_i) + \frac{\lambda}{n}  r(R \beta)}{\sum_j \ell(-y_j \beta \cdot x_j) + \lambda  \; r(R\beta)}&
        \\
        &\leq \frac{\ell(|\beta \cdot x_i|) + \frac{\lambda}{n}  r(R \beta)}
        {\sum_j \ell(-|\beta \cdot x_j|) + \lambda  r(R\beta)}&\mbox{As the loss function is nondecreasing}
        \\
        &\leq \frac{\ell(|\beta \cdot \beta | \frac{R}{\norm{\beta}_2}) + \frac{\lambda}{n}  r(R \beta)}
        {\sum_j \ell(-|\beta \cdot \beta | \frac{R}{\norm{\beta}_2}) + \lambda \; r(R\beta)}& \mbox{As maximum is when } x_i = \beta \frac{R}{\norm{\beta}_2}
        \\  
        &\leq \frac{\ell(R\norm{\beta}_2) + \frac{\lambda}{n}  r(R \beta)}
        {\sum_j \ell(-R\norm{\beta}_2) + \lambda \; r(R\beta)}& 
        \\
        &\leq 
        \frac{\ell(R \;\norm{\beta}_2)}
        {\sum_j \ell(-R \; \norm{\beta}_2)} + \frac{1}{n}&\mbox{As } a, b, c, d \ge 0 \mbox{ implies } \frac{a + b}{c+d} \le \frac{a}{c} + \frac{b}{d}
        \\
        &\leq 
        \frac{\ell(\sigma)}
        {\sum_j \ell(-\sigma)} + \frac{1}{n}&\mbox{By assumption }  R \norm{\beta}_2  < \sigma \\
        &\leq 
        \frac{\ell(\sigma)}
        {n \; \ell(-\sigma)} + \frac{1}{n}&
\end{align*}


Thus the sensitivity of every point is at most $ \frac{1}{\tau \lambda} + \frac{\ell(\sigma)}{n \; \ell(-\sigma)} + \frac{1}{n}$, and the total sensitivity $S$ is at most $\frac{n}{\tau \lambda} + \frac{\ell(\sigma)}{ \ell(-\sigma)} + 1$. The claim the follows by  Theorem \ref{Coreset_Construction_Theorem}.
\end{proof}

\begin{corollary}
 Let $(n, X, Y, \ell, r, \lambda, R, \kappa)$ be an instance
    of the RLM problem where the loss function $\ell$ is logistic regression
    or SVM, and  the regularizer $r$ is one of the 1-norm, 2-norm, or 2-norm squared. 
    Let $S' = \frac{12n}{\lambda} + 6 = 12 n ^{1-\kappa} + 6$. A  uniform sample of
    $q = \frac{10 S'}{\epsilon^2}((d+1) \log{S'} + \log(\frac{1}{\delta})))$
    points,  each with an associate weight of $u = \frac{n}{q}$,  is an $\epsilon$-coreset with probability at least
    $1-\delta$.
\end{corollary}

\begin{proof}
Since the VC-dimension of logistic regression and SVM is at most $d+1$, it is enough to show that the scaling condition holds in each case.
First consider logistic regression. Let $\sigma=1$. Then we have $l(-1) = \log(1+\exp(-1)) \neq 0$. In  the case
that $r(\beta) = \norm{\beta}_2$ it is sufficient
to take $\tau = \frac{1}{2}$ as $\ell(z) = \log (1 + \exp(z)) \le  2z$ when $z \ge 1$. Similarly its sufficient to take $\tau =\frac{1}{2}$ when
the regularizer is the 2-norm squared, 
as $\ell(z) = \log (1 + \exp(z)) \le  2z^2$ when $z \ge 1$.
As $\norm{\beta}_1 \ge \norm{\beta}_2$ it is also
sufficient to take $\tau =\frac{1}{2}$ when the regularizer is the 1-norm. Therefore, total sensitivity is bounded by $\frac{2n}{ \lambda} + 6$ in all of these cases.

Now consider SVM. 
Let  $\sigma= 1/2$. Then  $l(- 1/2) = 1/2 \neq 0$. 
In  the case
that $r(\beta) = \norm{\beta}_2$ it is sufficient
to take $\tau = \frac{1}{3}$ as $\ell(z) = 1 + z \le  3z$ when $z \ge \frac{1}{2}$; $\tau = \frac{1}{3}$ will be also sufficient when the regularizer is the 1-norm since $\norm{\beta}_1 \ge \norm{\beta}_2$.

Furthermore, if $\norm{\beta}_2\geq 1$, then $\norm{\beta}_2^2 \geq 4\norm{\beta}_2$; therefore, in the case that $r(\beta) = \norm{\beta}_2^2$, it is sufficient to take $\tau = \frac{1}{12}$. Therefore, total sensitivity is bounded by $\frac{12n}{ \lambda} + 4$.
\end{proof}

The implementation of uniform sampling, and the computation
of $R$, in the streaming
and MPC models is trivial. Uniform sampling and the computation
of $R$ in the relational model can be implemented without
joins because both can be expressed using {\it functional aggregate queries}, which can then be efficiently computed without joins~\cite{KhamisNR16}.

\section{Uniform Sampling Lower Bound}
\label{sec:lbuniform}

In this section we  show in Theorem \ref{thm:lb-uniform} that our analysis of uniform sampling is tight up to  poly-logarithmic factors.

\begin{theorem}
\label{thm:lb-uniform}
Assume that the loss function is either
logistic regression or SVM, and the regularizer is the 2-norm squared.  Let  $\epsilon,\gamma \in (0, 1)$ be arbitrary. 
For all sufficiently large $n$, there exists an instance $I_n$ of $n$ points such that
with probability at least $1-1/n^{\gamma/2}$ it 
will be the case that
for a uniform sample $C$ of $c = n^{1-\gamma}/\lambda = n^{1 - \kappa - \gamma}$ points, there is no weighting $U$
that will result in an $\epsilon$-coreset.
\end{theorem}
\begin{proof}
The instance $I_n$ consists
of points  located on the real line, so the dimension $d=1$.
A collection $A$ of  $n-(\lambda n^{\gamma/2})$ points
is located at $ +1$, and the remaining  $\lambda n^{\gamma/2}$
points are located at $-1$; call this collection of points $B$.
All points are labeled $+1$.
Note $R =1$. 

Let $C$ be the random sample of $c$ points, and $U$ an arbitrary
weighting of the points in $C$. Note that $U$ may depend on the
instantiation of $C$. Our goal is to show that with high probability, $(C, U)$ is not an $\epsilon$-coreset. Our proof strategy is to first show
that because almost all of the points are in $A$, it is likely that
$C$ contains only points from $A$. Then we want to show that,
 conditioned on $C \subseteq A$, that $C$ can not be a coreset
 for any possible weighting. We accomplish this by 
 showing that $\lim_{n \rightarrow \infty} H(\beta) = 1 $
 when $\beta = n^{\gamma/4}$.
 
We now show that one can use a standard union 
bound to establish that it is likely that $C \subseteq A$.
To accomplish this let $E_i$ be the probability that the
the $i^{th}$ point selected to be in $C$ is not in $A$. 
\begin{align*}
    \Pr[ C \subseteq A ] = 1- \Pr\Big[ \vee_{i \in C} E_i \Big]
    \geq 1- |C|\frac{|B|}{n} = 1 - \frac{n^{1-\gamma}}{\lambda}\frac{\lambda n^{\gamma/2}}{n} = 1-\frac{1}{n^{\gamma/2}}
\end{align*}

Now we show if $C \subseteq A$ and $n$ is large enough, then $(C,U)$ cannot be an $\epsilon$-coreset for any collection $U$ of weights.
To accomplish this consider the
the hypothesis $\beta_0 = n^{\gamma/4}$.
From the definition of coreset, it is sufficient to show
that $H(\beta_0)$, defined as, 
\begin{align}
    H(\beta_0) 
    &= \frac{|\sum_{i \in P} f_i(\beta_0)-\sum_{i \in C} u_i f_i(\beta_0)|}{\sum_{i \in P} f_i(\beta_0)}  \label{eq:Hdefn}
    \end{align}
is greater than $\epsilon$. We accomplish this by showing
that the limit as n goes to infinity of $H(\beta_0)$ is 1. 
Applying Condition \ref{obs:Un} we can conclude that
\begin{align}
    H(\beta_0) 
   \geq
    \frac{|\sum_{i \in P} \ell_i(\beta_0) -\sum_{i \in C} u_i \ell_i(\beta_0)| - \epsilon \lambda \norm{\beta_0}_2^2}{\sum_{i \in P} \ell_i(\beta_0) + \lambda \norm{\beta_0}_2^2}
\end{align}
Then, using the fact that $A$ and $B$ is a partition of the points and $C \subseteq A$ we can conclude that 
\begin{align}
    H(\beta_0)
    &\geq
    \frac{|\sum_{ i \in A} \ell_i(\beta_0) + \sum_{i \in B} \ell_i(\beta_0) -\sum_{i \in C} u_i \ell_i(\beta_0)| - \epsilon \lambda \norm{\beta_0}_2^2}{\sum_{i \in A} \ell_i(\beta_0) + \sum_{i \in B} \ell_i(\beta_0) + \lambda \norm{\beta_0}_2^2} \nonumber
    \\
    &=
    \frac{\left|\frac{\sum_{i \in A} \ell_i(\beta_0)}{\sum_{i \in B} \ell_i(\beta_0)}+1 -\frac{\sum_{i \in C} u_i \ell_i(\beta_0)}{\sum_{i \in B} \ell_i(\beta_0)}\right| - \frac{\epsilon \lambda \norm{\beta_0}_2^2}{\sum_{i \in B} \ell_i(\beta_0)}}{\frac{\sum_{i \in A} \ell_i(\beta_0)}{\sum_{i \in B} \ell_i(\beta_0)}+1 + \frac{\lambda \norm{\beta_0}_2^2}{\sum_{i \in B} \ell_i(\beta_0)}} \label{eqn:uniform_lb}
\end{align}

We now need to bound various terms in equation \eqref{eqn:uniform_lb}. Let us
first consider logistic regression. Note that 
\begin{align}
    \sum_{i \in B} \ell_i(\beta_0) = |B| \log(1+\exp(n^{\gamma/4})) \geq |B|n^{\gamma/4} = \lambda n^{3\gamma/4} \label{eqn:logistic-1}
\end{align}
Therefore, 
\begin{align}
\label{inequality_uniform_lowerbound_1}
    \lim \limits_{n\to \infty} \frac{\lambda \norm{\beta_0}_2^2}{\sum_{i \in B} \ell_i(\beta_0)} 
    &\leq 
    \lim \limits_{n\to \infty} \frac{\lambda n^{\gamma/2}}{\lambda n^{3\gamma/4}} = 0
\end{align}
Also note that
\begin{align}
\label{inequality_uniform_lowerbound_2}
    \lim \limits_{n\to \infty} \sum_{i \in A} \ell_i(\beta_0) = \lim \limits_{n\to \infty} |A| \log(1+\exp(-n^{\gamma/4})) \leq \lim \limits_{n\to \infty} n\exp(-n^{\gamma/4}) = 0
\end{align}
Finally, by Observation \ref{obs:Un}, we have, 
\begin{align}
\label{inequality_uniform_lowerbound_3}
    \lim \limits_{n\to \infty} \sum_{i \in C} u_i \ell_i(\beta_0) \leq \lim \limits_{n\to \infty} (1+\epsilon)n \exp(-n^{\gamma/4}) = 0
\end{align}

Combining equations \eqref{eqn:logistic-1}, \eqref{inequality_uniform_lowerbound_1}, \eqref{inequality_uniform_lowerbound_2}, and \eqref{inequality_uniform_lowerbound_3}, we the expression in equation
 \eqref{eqn:uniform_lb} converges to 1 as $n \to \infty$.
 Thus for sufficiently large $n$, $H(\beta_0) > \epsilon$and thus $(C, U)$ is not an $\epsilon$-coreset.

We now need to bound various terms in equation \eqref{eqn:uniform_lb} for SVM.
First note that
\begin{align}
    \sum_{i \in B} \ell_i(\beta_0) = |B| (1+n^{\gamma/4}) \geq |B|n^{\gamma/4} = \lambda n^{3\gamma/4} \label{eqn:svm-1}
\end{align}
Therefore, 
\begin{align}
\label{inequality_uniform_lowerbound_1_svm}
    \lim \limits_{n\to \infty} \frac{\lambda \norm{\beta_0}_2^2}{\sum_{i \in B} \ell_i(\beta_0)} 
    &\leq 
    \lim \limits_{n\to \infty} \frac{\lambda n^{\gamma/2}}{\lambda n^{3\gamma/4}} = 0
\end{align}
Also note that
\begin{align}
\label{inequality_uniform_lowerbound_2_svm}
    \lim \limits_{n\to \infty} \sum_{i \in A} \ell_i(\beta_0) = \lim \limits_{n\to \infty} |A| \max(0,1-n^{\gamma/4}) = 0
\end{align}
Finally, by Observation \ref{obs:Un}, 
we have that:
\begin{align}
\label{inequality_uniform_lowerbound_3_svm}
    \lim \limits_{n\to \infty} \sum_{i \in C} u_i \ell_i(\beta_0) \leq \lim \limits_{n\to \infty} (1+\epsilon)n \max(0,1-n^{\gamma/4}) = 0
\end{align}
Combining equations \eqref{eqn:svm-1}, \eqref{inequality_uniform_lowerbound_1_svm}, \eqref{inequality_uniform_lowerbound_2_svm}, and \eqref{inequality_uniform_lowerbound_3_svm}, we the expression in equation
 \eqref{eqn:uniform_lb} converges to 1 as $n \to \infty$.
 Thus  for sufficiently large $n$, $H(\beta_0) > \epsilon$and thus $(C, U)$ is not an $\epsilon$-coreset.
\end{proof}

\section{General Lower Bound on Coreset Size}
\label{sec:lbgeneral}

This section is devoted to proving the following theorem: 

\begin{theorem}
\label{thm:lb-general}
Assume that the loss function is either
logistic regression or SVM,  and the regularizer is the 2-norm squared.  Let  $\epsilon, \gamma \in (0, 1)$ be arbitrary. 
For all sufficiently large $n$, there exists an instance $I_n$ of $n$ points such that $I_n$ does not have  an $\eps$-coreset of size $O(n^{(1-\kappa)/5-\gamma})$
\end{theorem}

\subsection{Logistic Regression}

The goal in this subsection is to prove Theorem \ref{thm:lb-general} when the loss function is logistic regression. 
The lower bound instance $I_n$ consists of a collection of  $n$ positively-labeled points in $\Re^3$ uniformly spaced around a circle of  radius $1$ centered at $(0, 0, 1)$ in the plane  $z = 1$.  
 Note that  $R = \sqrt{2}$. 
 However for convenience, we will project $I_n$ down into a
 collection $X$ of points in the plane $z = 0$. So the resulting instance, which we call the
 circle instance, consists of $n$ points uniformly spread around
 the unit circle in $\Re^2$.  So for a hypothesis 
 $\beta = (\beta_x, \beta_y, \beta_z)$, $F(\beta)$ is now
$\sum_{x_i \in X} \ell(-y_i ((\beta_x, \beta_y) \cdot x_i + \beta_z)) +2  \lambda   \norm{\beta}_2^2$. So $\beta_z$ can be thought of as an offset or bias term, that allows hypotheses in $\Re^2$
that do not pass through the origin. 
 
  Fix a constant $c >0$ and a subset $C$ of $X$ that has size $k = c \frac{n^{1/5-\gamma}}{\lambda^{1/5}}= c n^{(1-\kappa)/5-\gamma}$ as a candidate coreset.
  Let $U$ be an arbitrary collection of associated weights. 
Toward finding a hypothesis that violates equation  \eqref{coreset_inequality}, define
a \emph{chunk} $A$ to be a collection of  $\frac{n}{4k}$ points in the middle of
  $\frac{n}{2k}$ consecutive points on the circle that are all not
  in $C$. So no point in the chunk $A$ is in $C$, and no point in the
  next $\frac{n}{8k}$ points in either direction around the
  circle are in $C$. 
  Its easy to observe that, by the pigeon principle, a chunk $A$ must exist. 
  Now let $\beta_A = (\beta_x, \beta_y, \beta_z)$ be the hypothesis where $(\beta_x, \beta_y) \cdot x_i + \beta_z = 0$ for the two points $x_i \in X \setminus A$ that are adjacent to the chunk $A$,
  that predicts $A$ incorrectly (and thus that predicts the points $X \setminus A$  correctly), and where $\norm{\beta_A}_2 = \sqrt{\frac{n^{1-\gamma}}{k\lambda}}$.
  To establish Theorem~\ref{thm:lb-general} we want to show that  equation \eqref{coreset_inequality} is not satisfied for the
  hypothesis $\beta_A$. 
By  Observation \ref{obs:Un} it is sufficient to show that the limit as $n \rightarrow \infty$ of:
     \begin{align*}
        \frac{\left|\sum\limits_{x_i \in X} \ell_i(\beta_A) -\sum\limits_{x_i \in C} u_i \ell_i(\beta_A)\right| - 2 \epsilon \lambda \norm{\beta_A}_2^2}{\sum\limits_{x_i \in X} \ell_i(\beta_A) + \lambda \norm{\beta_A}_2^2}
                =   
        \frac{\left|1 - \frac{\sum\limits_{x_i \in C} u_i \ell_i(\beta_A)}{\sum\limits_{x_i \in X} \ell_i(\beta_A)}\right| - \frac{2\epsilon \lambda \norm{\beta_A}_2^2}{\sum\limits_{x_i \in X} \ell_i(\beta_A)}}{1 + \frac{\lambda \norm{\beta_A}_2^2}{\sum\limits_{x_i \in X} \ell_i(\beta_A)}} 
 \end{align*}
 is 1. To accomplish this it is sufficient to show that the limits
 of the ratios in the second expression approach 0, which we do
 in the next two lemmas.

\begin{lemma}
\label{lemma_lowerbound_lim1}
	 $\lim \limits_{n\to \infty} \frac{\lambda \norm{\beta_A}_2^2}{\sum\limits_{x_i \in X} \ell_i(\beta_A)} = 0$.
\end{lemma}

\begin{proof} 
    As the $\frac{n}{4k}$ points in $A$ have been incorrectly
    classified by $\beta_A$, we know that $\ell_i(\beta_A) \geq \log{2}$ 
    for $x_i \in A$. Thus:
  \begin{align*}
  \lim \limits_{n\to \infty} \frac{\lambda \norm{\beta_A}_2^2}{\sum\limits_{x_i \in X} \ell_i(\beta_A)} \le
  \lim \limits_{n\to \infty} \frac{\lambda \frac{n^{1-\gamma}}{k\lambda}}{\frac{n}{4k} \log{2}}
        =\lim \limits_{n\to \infty}\frac{4}{n^{\gamma}\log{2}} = 0
\end{align*}
\end{proof}

Let $d_i$ be the distance between $x_i$ and  the line that passes through the first and last points in the
chunk $A$. 
Let $\theta_i$ be the angle formed by the the ray from the origin 
through  $x_i$ and the ray from
the origin to them middle point in $A$. 
Let $\theta= \max_{i \in A} \theta_i =\frac{2\pi}{n}\frac{n}{8k} = \frac{\pi}{4k}$. We then make two algebraic observations.

\begin{observation}
\label{lower_bound_observation_distance}
For all $x_i \in X$,	$d_i \norm{\beta_A}_2 /2\leq |(\beta_x, \beta_y) \cdot x_i + \beta_z| \leq d_i \norm{\beta_A}_2 $.
\end{observation}

\begin{proof}
	 	  It is well known that 
	  \begin{align*}
	        d_i = \frac{|(\beta_x, \beta_y) \cdot x_i + \beta_z|}{\sqrt{\beta_x^2 + \beta_y^2}}
	    \end{align*}
	    	    Therefore,
	    \begin{align*}
	       |(\beta_x, \beta_y) \cdot x_i + \beta_z|= d_i \sqrt{\beta_x^2 + \beta_y^2} \leq \norm{\beta_A} d_i
	    \end{align*}
	    Now we need to show $\norm{\beta_A} d_i/2 \leq |(\beta_x, \beta_y) \cdot x_i + \beta_z|$. Note that there are two points (points adjacent to $A$) $x_j = (a', b')$ for which $(\beta_x, \beta_y) \cdot x_j + \beta_z = 0$. Consider one of them. We have:
	    \begin{align*}
	        0 &=  \beta_x a' + \beta_y b' + \beta_z
	        \\
	        &\geq 
	        \beta_z - |\beta_x a' + \beta_y b'|
	        \\
	        &\geq
	        \beta_z - \sqrt{\beta_x^2 + \beta_y^2}\sqrt{a'^2+b'^2}
	    \end{align*}
	    Since the points are over a circle of size $1$ we have $\sqrt{a'^2+b'^2}=1$. Therefore,
	    \begin{align*}
	        \beta_x^2 + \beta_y^2 \geq \beta_z^2 
	    \end{align*}
	    So we can conclude:
	    \begin{align*}
	        |(\beta_x, \beta_y) \cdot x_i + \beta_z| &= d_i \sqrt{\beta_x^2 + \beta_y^2}
	              \geq \frac{d_i}{\sqrt{2}} \norm{\beta_A}
	        \geq \frac{d_i}{2} \norm{\beta_A}
	    \end{align*}
\end{proof}

\begin{observation}
\label{lower_bound_observation_cosine}
	For all $x_i \in X$,  $d_i = |\cos(\theta_i)-\cos(\theta)|$.
\end{observation}

\begin{lemma}
\label{lemma_lowerbound_lim2}
	$\lim\limits_{n\to \infty} \frac{\sum\limits_{x_i \in C} u_i \ell_i(\beta_A)}{\sum\limits_{x_i \in X} \ell_i(\beta_A)} = 0$.
\end{lemma}

\begin{proof}
We  have:
    \begin{align*}
        &~\lim \limits_{n\to \infty}\frac{\sum\limits_{x_i \in C} u_i \ell_i(\beta_A)}{\sum\limits_{x_i \in X} \ell_i(\beta_A)} &\\
        &= \lim \limits_{n\to \infty}
        \frac{\sum\limits_{x_i \in C} u_i \log\big(1+\exp(- ((\beta_x, \beta_y) \cdot x_i + \beta_z))\big)}{\sum\limits_{x_i \in X} \ell_i(\beta_A)}&
        \\
        &\leq \lim \limits_{n\to \infty}
        \frac{\sum\limits_{x_i \in C} u_i \log\big(1+\exp(-\frac{d_i\norm{\beta_A}_2}{2})\big)}{\sum\limits_{x_i \in A} \ell_i(\beta_A)} & \mbox{By Observation } \ref{lower_bound_observation_distance} \\
        &\leq \lim \limits_{n\to \infty}
        \frac{\sum\limits_{x_i \in C} u_i \log\big(1+\exp(-\frac{\norm{\beta_A}_2}{2}( \cos{\theta} - \cos{\theta_i}) )\big)}{\sum\limits_{x_i \in A} \ell_i(\beta_A)} & \mbox{By Observation }\ref{lower_bound_observation_cosine}
        \\ &\leq \lim \limits_{n\to \infty}
        \frac{\sum\limits_{x_i \in C} u_i \exp(-\frac{\norm{\beta_A}_2}{2}( \cos{\theta} - \cos{\theta_i}) )}{\sum\limits_{x_i \in A} \ell_i(\beta_A)}& \mbox{Since } \log(1+x) \leq x\\
         &\leq \lim \limits_{n\to \infty}
         \frac{\sum\limits_{x_i \in C} u_i \exp(-\frac{\norm{\beta_A}_2}{2}( \cos{\frac{\pi}{4k}} - \cos{\frac{\pi}{2k}}) )}{\sum\limits_{x_i \in A} \ell_i(\beta_A)}&\mbox{Since maximizer is when } \theta_i = \frac{\pi}{2k}
    \end{align*}

    Using the Taylor expansion of
        $\cos(x) = \sum\limits_{i=0}^{\infty} (-1)^i \frac{x^{2i}}{(2i)!} = 1-\frac{x^2}{2!}+\frac{x^4}{4!}-\dots$,
    we have $\cos(\frac{\pi }{4k})-\cos(\frac{\pi}{2k})
        \geq
        \frac{1}{2}\big((\frac{\pi}{2k})^2 - (\frac{\pi }{4k})^2\big) - O(\frac{1}{k^4})
        =(\frac{3\pi^2}{32k^2}) - O(\frac{1}{k^4})$.
     Plugging this inequality, we derive
     \begin{align*}
        \frac{\sum\limits_{x_i \in C} u_i \ell_i(\beta_A)}{\sum\limits_{x_i \in P} \ell_i(\beta_A)}
        &\leq
         \frac{\sum\limits_{x_i \in C} u_i \exp(-\frac{\norm{\beta_A}}{2}((\frac{3\pi^2}{32k^2}) - O(\frac{1}{k^4})) )}{\sum\limits_{x_i \in A} \ell_i(\beta_A)}
         \\
        &=
        \frac{\sum\limits_{x_i \in C} u_i \exp(-\frac{n^{2/5}}{2\sqrt c\lambda^{2/5}} (\frac{3\pi^2\lambda^{2/5}}{32 c^2 n^{2/5-2\gamma}} - O(\frac{\lambda^{4/5}}{n^{4/5-4\gamma}})))}{\sum\limits_{x_i \in A} \ell_i(\beta_A)}
        \\
        &=
        \frac{\sum\limits_{x_i \in C} u_i \exp(-\alpha n^{2\gamma} + O(\frac{\lambda^{2/5}}{n^{2/5-4\gamma}}))}{\sum\limits_{x_i \in A} \ell_i(\beta_A)},
    \end{align*}
        \noindent where $\alpha >0$ is a constant. Since all the points in $A$ are miss-classified, we have $\ell_i(\beta_A) \geq \log 2 $ for all of them. Using this fact and Observation \ref{obs:Un}, we have: 
    \begin{align*}
        \frac{\sum\limits_{x_i \in C} u_i \ell_i(\beta_A)}{\sum\limits_{x_i \in P} \ell_i(\beta_A)}
        &\leq
        \frac{(1+\epsilon) n \exp(-\alpha n^{2\gamma} + O(\frac{\lambda^{2/5}}{n^{2/5-4\gamma}}))}{\frac{n}{4k} \log{2}}
    \end{align*}
    
    Finally, using the fact that  $k = c \frac{n^{1/5-\gamma}}{\lambda^{1/5}}$ and taking the limit,  we conclude:
    \begin{align*}
    \lim \limits_{n \to \infty}
        \frac{\sum\limits_{x_i \in C} u_i \ell_i(\beta_A)}{\sum\limits_{x_i \in P} \ell_i(\beta_A)}
        &\leq
        \lim \limits_{n \to \infty} \frac{4k(1+\epsilon) \exp(-\alpha n^{2\gamma} + O(\frac{\lambda^{2/5}}{n^{2/5-4\gamma}}) )}{\log{2}} = 0
    \end{align*}
    \end{proof}

\subsection{SVM}

The goal in this subsection is to prove Theorem \ref{thm:lb-general} when the loss function is SVM. 
For the sake of contradiction, suppose an $\eps$-coreset $(C,u)$ of  size $k$ exists for the circle instance. 
We fix $A$ to be a  chunk. Similar to logistic regression, we set $\beta_A$ as the parameters of the linear SVM  that separates $A$ from $P/A$ such that the model predicts $A$ incorrectly and predicts the points $P/A$ as positive correctly and $\norm{\beta_A}_2 = \sqrt{\frac{n^{1-\gamma}}{k\lambda}} = \frac{n^{2/5}}{\sqrt{c} \lambda^{2/5}}$.

Our goal is to show Eqn. \eqref{coreset_inequality} tends to $1$ as $n$ grows to infinity. We can break the cost function of Linear SVM into two parts:
\begin{align*}
    F_{P,1}(\beta_A) := \sum\limits_{x_i \in P} \ell_i(\beta_A) + 2\lambda \norm{\beta_A}_2^2
\end{align*}
where $\ell_i(\beta_A) = \max(1-\beta_A x_i y_i,0) = 
\max( 1 - ((\beta_x, \beta_y) \cdot x_i + \beta_z) y_i, 0)$. Then, we determine the limit of the following quantities as $n$ grows to infinity.

\begin{lemma}
\label{lemma_lowerbound_svm_limits}
For the circle instance $P$, if $(C,u)$ is an $\epsilon$-coreset of $P$ with size $k = c \frac{n^{1/5-\gamma}}{\lambda^{1/5}}$ for linear SVM, and $A$ is a  chunk, then we have,

\begin{enumerate}
	\item $\lim \limits_{n\to \infty} \frac{\lambda \norm{\beta_A}_2^2}{\sum\limits_{x_i \in P} \ell_i(\beta_A)} = 0$;
	\item $\lim\limits_{n\to \infty} \frac{\sum\limits_{x_i \in C} u_i \ell_i(\beta_A)}{\sum\limits_{x_i \in P} \ell_i(\beta_A)} = 0$.
\end{enumerate}
\end{lemma}

Using this lemma, which we will prove soon, we can prove Theorem~\ref{thm:lb-general} for the linear SVM: The definition of coreset allows us to choose any $\beta$, so we can set $\beta = \beta_A$ for a  chunk $A$. Then, by Observation \ref{obs:Un}, Eqn. \eqref{coreset_inequality} simplifies to:
\begin{align*}
          \;\;\;\frac{|\sum_{x_i \in X} f_i(\beta_A)-\sum_{x_i \in C} u_i f_i(\beta_A)|}{\sum_{x_i \in  X} f_i(\beta_A)}
        &\geq
        \frac{|\sum_{x_i \in X} \ell_i(\beta_A) -\sum_{x_i \in C} u_i \ell_i(\beta_A)| - 2\epsilon \lambda \norm{\beta_A}_2^2}{\sum_{x_i \in X} \ell_i(\beta_A) + \lambda \norm{\beta_A}_2^2}
        \\
        &=
        \frac{|1 - \frac{\sum_{x_i \in C} u_i \ell_i(\beta_A)}{\sum_{x_i \in X} \ell_i(\beta_A)}| - \frac{2\epsilon \lambda \norm{\beta_A}_2^2}{\sum_{x_i \in X} \ell_i(\beta_A)}}{1 + \frac{\lambda \norm{\beta_A}_2^2}{\sum_{x_i \in  X} \ell_i(\beta_A)}}, 
\end{align*}
\noindent which tends to 1 as $n \rightarrow \infty$ by Lemma~\ref{lemma_lowerbound_svm_limits}. This implies that $(C, u)$ is not an $\eps$-coreset for the circle instance, which is a contradiction. This completes the proof of Theorem~\ref{thm:lb-general} for SVM. 

\subsubsection*{Proof of Lemma~\ref{lemma_lowerbound_svm_limits}}

The remainder of this section is devoted to proving Lemma~\ref{lemma_lowerbound_svm_limits}. The proof is very similar to the proof of Lemma~\ref{lemma_lowerbound_lim1} and \ref{lemma_lowerbound_lim2}.

\begin{proof} \emph{\bf of Claim 1 in Lemma~\ref{lemma_lowerbound_svm_limits}}
    We know for all points in $A$, $\ell_i(\beta_A) \geq 1$ this is because all of them have been incorrectly classified. We also know that since $A$ is a  chunk, $|A| = \frac{n}{4k}$.
    
    Therefore
    \begin{align*}
        \sum\limits_{x_i \in A} \ell_i(\beta_A) \geq \frac{n}{4k}
    \end{align*}
    We also know $\norm{\beta_A}_2 = \sqrt{\frac{n^{1-\gamma}}{k\lambda}}$, so we can conclude
    \begin{align*}
        \frac{\lambda \norm{\beta_A}_2^2}{\sum\limits_{x_i \in X} \ell_i(\beta_A)} 
        \leq
        \frac{\lambda \norm{\beta_A}_2^2}{\sum\limits_{x_i \in A} \ell_i(\beta_A)} 
        \leq
         \frac{\lambda \norm{\beta_A}_2^2}{\frac{n}{4k}}
        =
        \frac{\lambda \frac{n^{1-\gamma}}{k\lambda}}{\frac{n}{4k}}
        =\frac{4}{n^{\gamma}}
    \end{align*}
    The lemma follows by taking the limit of the above inequality.
\end{proof}

\begin{proof} \emph{\bf of Claim 2 in Lemma~\ref{lemma_lowerbound_svm_limits}}.
    Using Observation \ref{lower_bound_observation_distance} and the fact that all the points in the coreset are predicted correctly by $\beta_A$ we have:
    \begin{align*}
        \frac{\sum\limits_{x_i \in C} u_i \ell_i(\beta_A)}{\sum\limits_{x_i \in X} \ell_i(\beta_A)} 
        \leq
        \frac{\sum\limits_{x_i \in C} u_i \max\big(0,1-\frac{\norm{\beta_A}d_i}{2}\big)}{\sum\limits_{x_i \in A} \ell_i(\beta_A)}
    \end{align*}
    Then, by Observation~\ref{lower_bound_observation_cosine} we have 
    \begin{align*}
        \frac{\sum\limits_{x_i \in C} u_i \ell_i(\beta_A)}{\sum\limits_{x_i \in X} \ell_i(\beta_A)} 
        &\leq
        \frac{\sum\limits_{x_i \in C} u_i \max\big(0,1-\frac{\norm{\beta_A}}{2}( \cos{\theta} - \cos{\theta_i}) \big)}{\sum\limits_{x_i \in A} \ell_i(\beta_A)}
    \end{align*}

    By definition of  chunk, we know all the points in $C$ are at least $\frac{n}{4k}$ away from the center of $A$, which means the closest point in $C$ to chunk $A$ is at least $\frac{n}{8k}$ points away, we have $\theta_i \geq \theta + \frac{2\pi}{n}\frac{n}{8k} = \frac{\pi}{2k}$. Therefore, 
    \begin{align*}
        \frac{\sum\limits_{x_i \in C} u_i \ell_i(\beta_A)}{\sum\limits_{x_i \in X} \ell_i(\beta_A)}
        &\leq
         \frac{\sum\limits_{x_i \in C} u_i \max\big(0,1-\frac{\norm{\beta_A}}{2}( \cos{\frac{\pi}{4k}} - \cos{\frac{\pi}{2k}}) \big)}{\sum\limits_{x_i \in A} \ell_i(\beta_A)}
    \end{align*}
    Using the Taylor expansion of
        $\cos(x) = \sum\limits_{i=0}^{\infty} (-1)^i \frac{x^{2i}}{(2i)!} = 1-\frac{x^2}{2!}+\frac{x^4}{4!}-\dots$, we have, 
    \begin{align*}
        \cos(\frac{\pi }{4k})-\cos(\frac{\pi}{2k})
        &\geq
        \frac{1}{2}\big((\frac{\pi}{2k})^2 - (\frac{\pi }{4k})^2\big) - O(\frac{1}{k^4})
        =(\frac{3\pi^2}{32k^2}) - O(\frac{1}{k^4})
    \end{align*}
     Therefore, we derive, 
     \begin{align*}
        \frac{\sum\limits_{x_i \in C} u_i \ell_i(\beta_A)}{\sum\limits_{x_i \in X} \ell_i(\beta_A)}
        &\leq
         \frac{\sum\limits_{x_i \in C} u_i \max\big(0,1-\frac{\norm{\beta_A}}{2}((\frac{3\pi^2}{32k^2}) - O(\frac{1}{k^4})) \big)}{\sum\limits_{x_i \in A} \ell_i(\beta_A)}
         \\
        &=
        \frac{\sum\limits_{x_i \in C} u_i \max\big(0,1-\frac{n^{2/5}}{2c\lambda^{2/5}} (\frac{3\pi^2\lambda^{2/5}}{32 c n^{2/5-2\gamma}} - O(\frac{\lambda^{4/5}}{n^{4/5-4\gamma}}))\big)}{\sum\limits_{x_i \in A} \ell_i(\beta_A)}
        \\
        &=
        \frac{\sum\limits_{x_i \in C} u_i \max\big(0,1-\alpha n^{2\gamma} + O(\frac{\lambda^{2/5}}{n^{2/5-4\gamma}})\big)}{\sum\limits_{x_i \in A} \ell_i(\beta_A)}
    \end{align*}
    For large enough $n$, we have $\max\big(0,1-\alpha n^{2\gamma} + O(\frac{\lambda^{2/5}}{n^{2/5-4\gamma}})\big) = 0$.  Therefore, by taking the limit we have:
    \begin{align*}
    \lim \limits_{n \to \infty}
        \frac{\sum\limits_{x_i \in C} u_i \ell_i(\beta_A)}{\sum\limits_{x_i \in X} \ell_i(\beta_A)}
        &\leq
        \lim \limits_{n \to \infty}  \frac{\sum\limits_{x_i \in C} u_i \max\big(0,1-\alpha n^{2\gamma} + O(\frac{\lambda^{2/5}}{n^{2/5-4\gamma}})\big)}{\sum\limits_{x_i \in A} \ell_i(\beta_A)} = 0
    \end{align*}
\end{proof}



\section{Experiments}
\label{sec:experiments}

We next experimentally evaluate the practical utility of our uniform sampling scheme for logistic regression.  Using several real-world datasets from the UCI machine learning dataset repository~\cite{ucimlrepository}, we uniformly generate samples  of different sizes and train a logistic regression model.

Logistic regression models are trained using mlpack~\cite{mlpack2018}.  Given $\lambda = 0.1$, we sweep sample sizes from 50 points up to the dataset size, reporting the mean approximation $H(\beta)$ (according to Eqn.~\ref{coreset_inequality}) in Figure~\ref{fig:approx_sweep} for a variety of datasets.    When training the models, the L-BFGS optimizer is used until convergence~\cite{liu1989limited}.  However, although this works for our experiments, note that in general it is not feasible to use L-BFGS like this, specifically when datasets are very large, or when we are in restricted computation access models, as we have considered in this paper.  This is because a single L-BFGS step requires computation of the gradient of $f_i(\beta)$ for {\it every} $i \in [n]$.

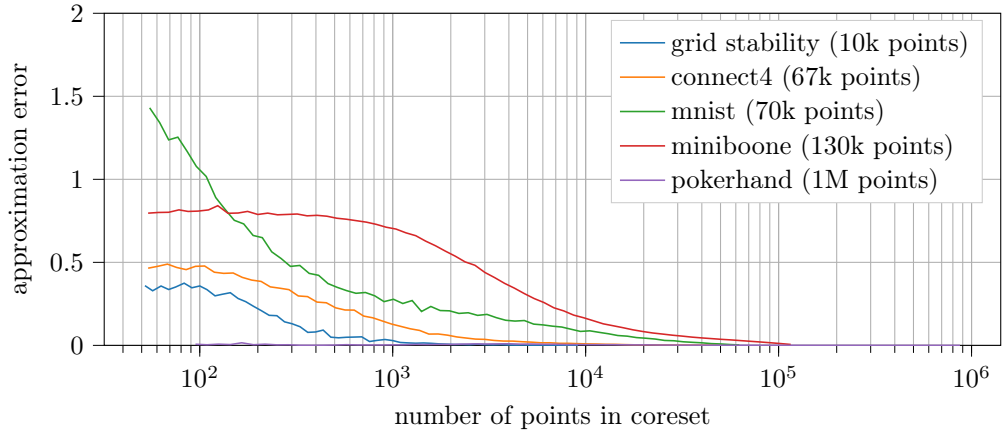
\begin{figure}[t!]
\begin{tikzpicture}

\definecolor{color1}{rgb}{1,0.498039215686275,0.0549019607843137}
\definecolor{color0}{rgb}{0.12156862745098,0.466666666666667,0.705882352941177}
\definecolor{color3}{rgb}{0.83921568627451,0.152941176470588,0.156862745098039}
\definecolor{color2}{rgb}{0.172549019607843,0.627450980392157,0.172549019607843}
\definecolor{color4}{rgb}{0.580392156862745,0.403921568627451,0.741176470588235}

\begin{axis}[
height=6cm,
legend cell align={left},
legend style={draw=white!80.0!black},
log basis x={10},
tick align=outside,
tick pos=left,
width=13.5cm,
x grid style={lightgray!92.0261437908!black},
xlabel={number of points in coreset},
xmajorgrids,
xmin=31.9744825941491, xmax=1416444.37456159,
xminorgrids,
xmode=log,
xtick style={color=black},
y grid style={lightgray!92.0261437908!black},
ylabel={approximation error},
ymajorgrids,
ymin=0, ymax=2,
yminorgrids,
ytick style={color=black}
]
\addplot [semithick, color0]
table {%
52 0.359592091989915
57 0.328815561251234
63 0.357379660774134
69 0.336367170355427
75 0.351423461192007
83 0.37449309884819
91 0.346801944803498
100 0.357803129200129
109 0.335660047170245
120 0.298302227247578
131 0.307932584878555
144 0.317169471299332
158 0.28223577439482
173 0.263306007293391
190 0.235138727156012
208 0.209723104557325
229 0.180900366020583
251 0.178644013823065
275 0.141759615798614
301 0.130108338253812
331 0.112252026894583
363 0.0782625850274832
398 0.08063693427789
436 0.0920323901456352
478 0.0493274247267471
524 0.0455559909087972
575 0.0488104279311002
630 0.0499241883610042
691 0.0513761004386104
758 0.0231926568526852
831 0.0301817580130738
912 0.0353828273709102
1000 0.0287877479064099
1096 0.0173439087455302
1202 0.0148650386289563
1318 0.0129736717331838
1445 0.0144674965207778
1584 0.0123816578332134
1737 0.0091420904200574
1905 0.00733602242728691
2089 0.00532557789212775
2290 0.00549682175597601
2511 0.00405058939629606
2754 0.00745994589554438
3019 0.00421036651310406
3311 0.00294868153191883
3630 0.0044483493152697
3981 0.00229562212486073
4365 0.00165287503764382
4786 0.00205285513007816
5248 0.00103796241199264
5754 0.00114198287217753
6309 0.00113828210050524
6918 0.00077926802170217
7585 0.000502553506516466
8317 0.000401002400402408
9120 0.000186000327585438
};
\addlegendentry{grid stability (10k points)}
\addplot [semithick, color1]
table {%
54 0.464786698783158
61 0.476843073334072
68 0.489336523581206
76 0.469012265192832
85 0.456234313735622
95 0.476001063932376
106 0.477959150979499
119 0.440505953730935
133 0.433681650623377
149 0.435837840486384
166 0.409591397473367
186 0.393033882300162
208 0.385121731193419
232 0.352241136310322
259 0.344532242804644
290 0.3355598140361
324 0.297643682350565
362 0.293658882637427
405 0.26074068485245
453 0.256533221636285
506 0.226075275656254
566 0.213126910601933
632 0.212317637036509
707 0.17650384294759
790 0.165780801601832
883 0.146833224115342
987 0.128295079708552
1103 0.114130616322292
1233 0.0998861056600723
1378 0.0899929462359139
1540 0.0679766386876022
1721 0.0685431108835169
1923 0.0615073088703266
2150 0.050759059154202
2403 0.0441648924241431
2685 0.0377791095158805
3001 0.0361060973161576
3354 0.0296174560571269
3749 0.0259268303520363
4190 0.0255713751676246
4683 0.0219929731608982
5234 0.0193876216564798
5849 0.0155123541685361
6537 0.0156452306049072
7306 0.0123531833759827
8166 0.0122476855742666
9127 0.0102849709881995
10200 0.00885411232922527
11400 0.00813235761895279
12741 0.00646608494311965
14240 0.00616257967095942
15915 0.00461840906321784
17787 0.00346652076442906
19879 0.00349324532915678
22217 0.003339343084145
24831 0.00258204691771745
27752 0.00183995275833412
31016 0.00175595010811507
34664 0.00127919224187878
38742 0.00107687289936697
43299 0.000796133664753255
48392 0.000578977724928193
54085 0.000309060835674269
60446 0.000170651098315807
};
\addlegendentry{connect4 (67k points)}
\addplot [semithick, color2]
table {%
55 1.43057532916069
62 1.34115483466109
69 1.23760738457989
77 1.25384903811533
86 1.16837179243273
96 1.07784850258193
108 1.01726774768673
121 0.890215091113543
135 0.82022064296618
151 0.753171687278135
169 0.731346182099309
189 0.662011492612995
211 0.648595758930189
236 0.563381156345557
264 0.523091368625698
295 0.475670306637689
330 0.481272978653566
369 0.433512316325394
413 0.421889210793694
462 0.37213057861028
516 0.349298295391658
577 0.329642822460765
645 0.313625988944619
722 0.318421244414661
807 0.298420682964942
902 0.263379072361214
1009 0.277735405567576
1128 0.252117802109388
1261 0.269969746221453
1410 0.203969036312064
1576 0.234296820399705
1762 0.209756831522347
1970 0.208037499358184
2203 0.19248985057793
2463 0.196403790682268
2754 0.180689595711421
3079 0.186616866432581
3442 0.166936729667209
3849 0.151219112213001
4303 0.14626834429298
4811 0.149554441912848
5379 0.127938557485962
6014 0.123269206932396
6723 0.115712579182869
7517 0.110400930890757
8404 0.0971094576168111
9396 0.0840147253389465
10505 0.0875773786442042
11745 0.0784597934487425
13132 0.069829592938848
14681 0.0592596835413331
16414 0.0554754838624876
18352 0.0468683285269526
20518 0.0429396032025141
22939 0.0380458634557016
25647 0.0296480511117071
28674 0.0277406239039781
32058 0.0215289594615267
35841 0.018274591908571
40072 0.0153142801311345
44801 0.0105311571185056
50089 0.00848762747756939
56001 0.00506757568704415
62610 0.00233141370085459
};
\addlegendentry{mnist (70k points)}
\addplot [semithick, color3]
table {%
54 0.796251924361825
61 0.800403769345765
69 0.801466918974309
78 0.816068740050724
87 0.806598940346943
98 0.808611779097848
111 0.815739577046954
124 0.841190806015008
140 0.79543561377549
158 0.797062327623066
177 0.80674882310457
200 0.788350752912874
225 0.796368987739104
253 0.786286796082943
284 0.788546679981824
320 0.79105923890658
360 0.780140805212449
405 0.783166171313811
456 0.777846675742671
513 0.765761327977093
577 0.759511402736765
649 0.750832415489254
731 0.74202213847186
822 0.728767626063804
925 0.711567137962314
1040 0.70046579551064
1170 0.677853511293648
1317 0.660493297530454
1481 0.627925605571229
1666 0.599945049992716
1875 0.567542854038986
2109 0.537278231658245
2373 0.502714828403505
2669 0.482292851161229
3003 0.439288005017936
3378 0.4078915435527
3801 0.37347365877984
4276 0.344439170679452
4810 0.311804458160227
5411 0.281923563857586
6088 0.257646015927681
6848 0.2261117051549
7704 0.207482746740284
8667 0.18216881259732
9750 0.166506055579817
10969 0.148188068948769
12340 0.129667611798241
13882 0.118118038993012
15617 0.106394387466998
17569 0.0937063465565669
19764 0.0835408663551756
22234 0.0746205051709098
25013 0.0681358518675357
28139 0.0612873450011247
31656 0.0557187012932793
35612 0.0495849870278914
40062 0.0447797442741186
45069 0.040556494041552
50701 0.036201781938298
57038 0.0326634057722954
64166 0.0283038212250719
72185 0.024113193114532
81206 0.0199811293256994
91355 0.0154501157153083
102771 0.0106449265006868
115615 0.00546288526708791
};
\addlegendentry{miniboone (130k points)}
\addplot [semithick, color4]
table {%
95 0.00680713313650728
109 0.00363341677443715
125 0.0072750493735118
144 0.00480033012793573
165 0.0155481800719102
190 0.00384275578972254
218 0.00716974298617276
251 0.00307567133231487
288 0.00363765434068296
331 0.00066408460403195
380 0.000512232668885574
436 0.000898982242843394
501 0.00186794887728848
575 0.00405369518069614
660 0.00243959214492897
758 0.00233609573622108
870 0.00227065506569507
1000 0.00310551914048363
1148 0.00364943676873935
1318 0.005949691547057
1513 0.00499975041901055
1737 0.0060013231321887
1995 0.005723737191797
2290 0.00747140490168025
2630 0.00900270284819449
3019 0.00828626886956269
3467 0.00907165217561737
3981 0.00973896540515375
4570 0.00739608154400945
5248 0.00824652002984576
6025 0.00651907772206248
6918 0.00678832343350545
7943 0.00503871419738871
9120 0.00495107623974963
10471 0.00434023945288107
12022 0.00418679355756492
13803 0.00411047606145657
15848 0.00328779222164197
18197 0.00293014505167393
20892 0.00302655641766056
23988 0.00222409089615644
27542 0.0018947836255362
31622 0.00145082584141336
36307 0.00135964303879395
41686 0.00119899964553499
47863 0.000953427710298205
54954 0.000989749759392676
63095 0.00086051321074236
72443 0.000675817215360309
83176 0.000553707237913353
95499 0.0004991852812921
109647 0.000423180750034874
125892 0.000371134985819553
144543 0.000299140480041045
165958 0.000228370694660189
190546 0.000258736719512709
218776 0.000213650975557594
251188 0.000131474823917132
288403 0.000113490312460538
331131 0.000122732697341676
380189 7.89894706816991e-05
436515 7.27573865112434e-05
501187 5.84395924294632e-05
575439 4.1963112953742e-05
660693 2.5113041444273e-05
758577 1.92607158118732e-05
870963 7.61754807766052e-06
};
\addlegendentry{pokerhand (1M points)}
\end{axis}

\end{tikzpicture}
\vspace*{-0.3em}
\caption{sample approximation error vs. sample size for different datasets.}
\label{fig:approx_sweep}
\end{figure}

In extremely-large-data or streaming settings, a typical strategy for training a logistic regression model is to use mini-batch SGD~\cite{ruder2016overview}, where the model's parameters are iteratively updated using a gradient computation on a small batch of random points.  However, SGD-like optimizers can converge very slowly in practice and have a number of parameters to configure (learning rate, number of epochs, batch size, and so forth).  But because our theory allows us to choose a sufficiently small sample, we can use a full-batch optimizer like L-BFGS and this often converges to a much better solution orders of magnitude more quickly.

To demonstrate this, we train a logistic regression model on a sample using L-BFGS and on the full dataset using SGD for 20 epochs.  At each step of the optimization, we record the wall-clock time and compute the loss on the full training set (the loss computation time is not included in the wall-clock time).  Figure~\ref{fig:learning_curves} shows the results for three trials of each strategy on two moderately-sized datasets.  It is clear from these results that a full-batch gradient descent technique can provide a good approximation of the full-dataset model with orders-of-magnitude speedup; in fact, L-BFGS is often able to recover a much better model than even 20 epochs of SGD!

\begin{figure}[t!]
    \centering
    \begin{subfigure}[t]{0.48\textwidth}
        \centering
        \input{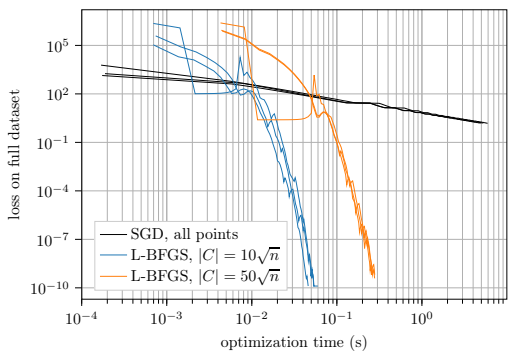}
        \caption{{\tt covertype} dataset (581k points).}
    \end{subfigure}
    \begin{subfigure}[t]{0.48\textwidth}
        \centering
        \input{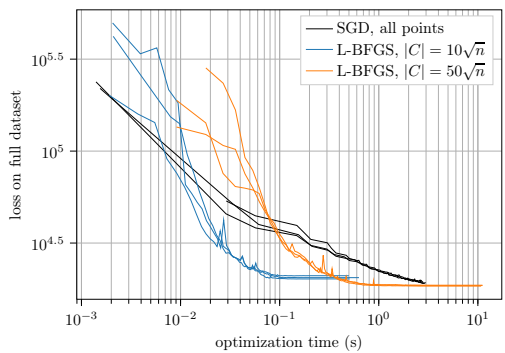}
        \caption{{\tt mnist} dataset (70k points).}
    \end{subfigure}
    \caption{Learning curves; log-log axes.  Three trials of each strategy are shown.  Note the orders-of-magnitude faster convergence for L-BFGS on samples.}
    \label{fig:learning_curves}
\end{figure}

To elaborate our experiments we used 5 datasets from the UCI dataset repository~\cite{ucimlrepository} of various sizes and dimensionalities.  These datasets are collected from synthetic and real-world data sources, and so represent a reasonable collection of diverse datasets.  Table~\ref{tab:datasets} gives details on the number of points ($n$) and the number of dimensions ($d$) for each dataset.

\begin{table}[h]
\begin{center}
\begin{tabular}{lllccc}
\toprule
 & & & \multicolumn{3}{c}{coreset error $\overline{H}(\beta)$} \\
{\bf dataset} & $n$ & $d$ & $|C| = \sqrt{n}$ & $|C| = 10\sqrt{n}$ & $|C| = 20\sqrt{n}$ \\
\midrule
{\small \tt connect4} & 67557 & 126 & $0.35 \pm 0.01$ & $0.04 \pm 0.00$ & $0.02 \pm 0.00$ \\
{\small \tt grid\_stability} & 10000 & 12 & $0.52 \pm 0.00$ & $0.02 \pm 0.00$ & $0.01 \pm 0.00$ \\
{\small \tt miniboone} & 130064 & 50 & $0.78 \pm 0.01$ & $0.39 \pm 0.00$ & $0.22 \pm 0.00$ \\
{\small \tt mnist} & 70000 & 784 & $0.53 \pm 0.02$ & $0.19 \pm 0.02$ & $0.14 \pm 0.00$ \\
{\small \tt pokerhand} & 1000000 & 85 & $1.01 \pm 0.00$ & $0.00 \pm 0.00$ & $0.00 \pm 0.00$ \\
\bottomrule
\end{tabular}
\end{center}
\caption{Dataset information and relative approximation of logistic regression objective with coresets $C$ of different sizes.  Three trials are used.  Coreset error of 0 indicates a very good approximation.}
\label{tab:datasets}
\end{table}

In addition, we run three trials of uniform sampling with three different coreset sizes: $\sqrt{n}$, $10\sqrt{n}$, and $20\sqrt{n}$.  We plot the relative difference in loss measures (0 means a perfect approximation).  Specifically, the approximation given in the table, $H(\beta)$, is given as

\begin{equation}
    H(\beta) = \frac{\left| \sum_{i = 1}^{n} f_i(\beta) - \sum_{i \in C} u_i f_i(\beta) \right|}{\sum_{i = 1}^{n} f_i(\beta)}.
\end{equation}

We report the mean of $H(\beta)$ over three trials: $\overline{H}(\beta)$.

Overall, we can see that our theory is empirically validated: uniform sampling provides samples that give good empirical approximation, and the use of these samples can result in orders-of-magnitude speedup for learning models.  Thus, our theory shows and our experiments justify that uniform sampling to obtain coresets is a compelling and practical approach for restricted access computation models.

\section*{Acknowledgements}

We thank Hung Ngo, Long Nguyen, and Zhao Ren for helpful discussions. 

\bibliographystyle{plain}
\bibliography{logistic}

\end{document}